\documentclass{article}

\usepackage{microtype}
\usepackage{graphicx}
\usepackage{subfig}
\usepackage{booktabs} % for professional tables
\usepackage{hyperref}
\usepackage{amsthm,amsfonts,amsmath}
\usepackage{verbatim}
\newtheorem*{theorem*}{Theorem}

\newtheorem*{lemma*}{Lemma}
\newtheorem{lemma}{Lemma}
\usepackage[accepted]{icml2019}

\icmltitlerunning{Learning State Abstractions for Transfer in Continuous Control}

\newcommand{\p}[0]{\textrm{Pr}}
\newcommand{\FancyS}{\mathcal{S}}

\newcommand{\FancyC}{\mathcal{C}}
\newcommand{\FancyF}{\mathcal{F}}
\newcommand{\FancyG}{\mathcal{G}}
\newcommand{\FancyX}{\mathcal{X}}

\newcommand{\piPhi}{\pi^{\phi}}

\DeclareMathOperator*{\argmax}{arg\,max}

\newcommand{\rademacher}[1]{\text{Rad}(#1)}
\newcommand{\expected}[2]{\mathbb{E}_{#1}\Big[ #2 \Big]}

\newcommand{\elone}[1]{\left\lVert#1\right\rVert_1}
\begin{document}

\twocolumn[
\icmltitle{Learning State Abstractions for Transfer in Continuous Control}



\icmlsetsymbol{equal}{*}

% -- Author Info --
\begin{icmlauthorlist}
\icmlauthor{Kavosh Asadi}{equal,b}
\icmlauthor{David Abel}{equal,b}
\icmlauthor{Michael Littman}{b}
\end{icmlauthorlist}
\icmlaffiliation{b}{Department of Computer Science, Brown University}
\icmlcorrespondingauthor{Kavosh Asadi}{k8@brown.edu}

% Keywords
\icmlkeywords{Reinforcement Learning, Continuous Control, MDP, Abstraction, State Abstraction}

\vskip 0.3in
]

\printAffiliationsAndNotice{\icmlEqualContribution} % otherwise use the standard text.

% --- Abstract ---
\begin{abstract}
% Main question.
Can simple algorithms with a good representation solve challenging reinforcement learning problems?
In this work, we answer this question in the affirmative, where we take ``simple learning algorithm" to be tabular Q-Learning, the ``good representations" to be a learned state abstraction, and ``challenging problems" to be continuous control tasks.
% (dnote) Alternative: Our main contribution is a learning algorithm that abstract a continuous state-space into a discrete one.
% Old: Our main contribution is a learning algorithm that turns a continuous state-space into a discrete abstract space
Our main contribution is a learning algorithm that abstracts a continuous state-space into a discrete one. We transfer this learned representation to unseen problems to enable effective learning. We provide theory showing that learned abstractions maintain a bounded value loss, and we report experiments showing that the abstractions empower tabular Q-Learning to learn efficiently in unseen tasks. 
\end{abstract}

% --------------------
% --- Introduction ---
% --------------------
\section{Introduction}

% Big Picture
Finding the right representation is critical for effective reinforcement learning (RL). A domain like Backgammon can be tractable given a simple representation~\cite{tesauro1995temporal}, but replace this typical input stream with a 3D point cloud of a Backgammon board, and the problem becomes intractable~\cite{konidaris2019necessity}. In this context, representation learning could be thought of as a learning process that enables faster learning in future. In this paper we explore the limits of this reasoning. Specifically, we study whether unseen continuous control problems can be solved by simple tabular RL algorithms using a state-abstraction function learned from previous problems.

% Related work.
Identifying useful abstractions has long been a goal of AI and RL; indeed, understanding abstraction's role in intelligence was one of the areas of inquiry of the 1956 Dartmouth Summer AI Workshop~\cite{mccarthy2006proposal}, and has been an active area of study since~\cite{brooks1991intelligence,dayan1993feudal,singh1995reinforcement,sutton1999between,parr1998reinforcement,dietterich2000hierarchical,li2006towards}. We narrow this study by concentrating on state abstractions that translate a continuous state space into a small discrete one, suitable for use with ``tabular" learning algorithms which are simple and well-understood \cite{sutton2018reinforcement}. Prior work has studied a similar setting leading to tree-based methods for abstracting continuous state spaces~\cite{moore1994parti,uther1998tree,feng2004dynamic,menashe2018state}, algorithms for finding state abstractions suitable for transfer~\cite{walsh2006transferring,cobo2011automatic,cobo2012automatic,abel2018salrl}, and methods for learning succinct representations of continuous states~\cite{whiteson2007adaptive,jonschkowski2015learning}.

% Page one figure.
\begin{figure}%
    \centering
    \subfloat{{\includegraphics[width=0.42\columnwidth]{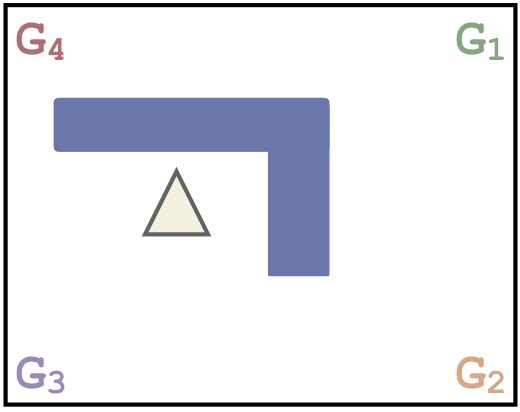} }} \hspace{3mm}
    %\qquad
    %\subfloat{{\includegraphics[width=4.25cm]{figures/puddle/puddle_multi.jpg}} }%
    \subfloat{{\includegraphics[width=0.42\columnwidth]{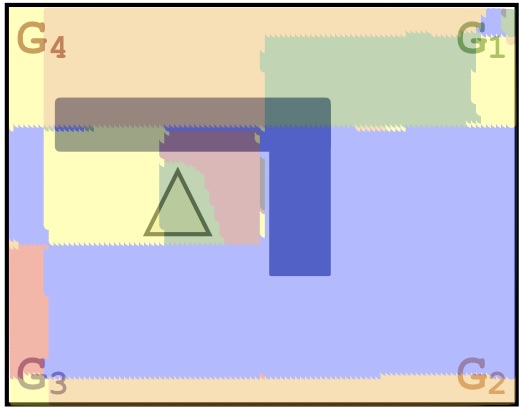}} }%
    \caption{Learned state abstractions (right) in the Puddle World domain (left), where goal location is one of the four corners. The abstraction function is trained using three randomly chosen problems, and is then tested to solve the fourth problem. Notice that states that generally have the same optimal policy are clustered together in a common abstract state.}%
    \label{fig:example}%
\end{figure}

% Our work.
More concretely, we introduce and analyze an algorithm for learning a state abstraction. We then transfer the learned abstraction to help perform efficient reinforcement learning in unseen problems. Notably, this algorithm is well suited to continuous domains---after training, it outputs a state abstraction that maps continuous states into a small, finite state space, suitable for use with tabular RL algorithms. We present initial analysis and support for this abstraction-learning algorithm, emphasizing its ability to enable downstream tabular RL algorithms to solve new problems not seen during the training of the abstraction. 

We provide a theorem on the sample complexity of learning the abstraction function, allowing us to relate the value loss under the learned abstraction function to both the number of samples used for training the abstraction function and to the Rademacher complexity of the set of abstraction functions in the hypothesis space. Moreover, our central experimental finding is that the learned state abstraction enables tabular Q-learning to perform sample efficient reinforcement learning in problems with continuous state spaces. 

An example of a learned abstraction function in the Puddle World domain is shown in Figure \ref{fig:example}. In section \ref{section:explain_algorithm} we detail the algorithm used to learn these abstractions.

% ------------------
% --- Background ---
% ------------------
\section{Background}

We begin with background on RL and state abstraction.

% RL Background.
We take the standard treatment of RL~\cite{sutton2018reinforcement}: An agent learns to make decisions that maximize reward in an environment through interaction alone. We make the usual assumption that the environment can be modeled by a Markov Decision Process (MDP)~\cite{puterman2014markov}.

% State Abstraction.
State abstraction describes methods for reducing the size of an MDP's state space, typically by aggregating states together in the environmental MDP. With a smaller state space that preserves some characteristics, RL algorithms can often learn to make good decisions from fewer samples and with less computation.

More formally, a state abstraction is a function, $\phi : \FancyS \rightarrow \FancyC$, that maps each ground state, $s \in \FancyS$, into an abstract state, $c \in \FancyC$. Usually, such state abstractions are chosen so that $|\FancyC| \ll |\FancyS|$, thus the state space is more tractable to work with (either by making exploration, planning, or other aspects of decision making easier).

When the environment's underlying state space is continuous however, a natural family of relevant abstractions are those that translate the continuous state space into a discrete one. Such functions can dramatically simplify problems that are otherwise intractable for certain types of RL algorithms~\cite{moore1994parti,uther1998tree,lee2004adaptive}, like  traditional Q-Learning~\cite{watkins1992q} and TD-Learning~\cite{sutton1988learning}.

In order to study the sample complexity of our abstraction learning algorithm we use a tool from statistical learning theory referred to as Rademacher complexity \cite{bartlett2002rademacher,mohri2018foundations}. Consider a function $f:\FancyS\mapsto [-1,1]$, and an arbitrarily large set of such functions $\FancyF$. We define Rademacher complexity of this set, $\rademacher{\FancyF}$, as follows:
$$\rademacher{\FancyF}:=\expected{X_j,\sigma_j}{\sup_{f\in\FancyF}\frac{1}{n}\sum_{j=1}^n \sigma_j f(X_j)}\ ,$$
where $\sigma_j$, referred to as Rademacher random variables, are drawn uniformly at random from $\{\pm 1\}$. One could think of these variables as independent and identically distributed noise. With this perspective, the average $\frac{1}{n}\sum_{j=1}^n \sigma_j f(X_j)$ can be thought of as the covariance between $f(\cdot)$ and noise, or in other words, how well $f(\cdot)$ matches the noise. If for any realization of noise there exists an $f\in\FancyF$ for which this average is large, then we can accurately learn noise, and the Rademacher complexity is large.

We can extend the previous Rademacher definition to vector-valued functions that map to $[-1,1]^{m}$. Imagine a function $g:=\langle f_1,...,f_m \rangle$ where $\forall\ i\ f_i \in \FancyF$. Define the set of such functions $\FancyG$. We similarly define the Rademacher complexity of $\FancyG$ as follows:
$$\rademacher{\FancyG}:=\expected{X_j,\sigma_{ji}}{\sup_{g\in\FancyG}\frac{1}{n}\sum_{j=1}^n\sum_{i=1}^m \sigma_{ji} g(X_j)_i}\ ,$$
where $\sigma_{ji}$ are drawn uniformly randomly from $\{\pm 1\}$.

% ------------------
% -- Related Work --
% ------------------
\section{Related Work}

% Learning rules/representation.
%Prior work has explored the relationship between choice of representations and choice of learning rules. \dnote{I was imagining referencing the Thomas/Brunskill one}

We now summarize relevant prior literature. 

% Tree based.
We study state abstractions that map from continuous states to discrete ones---in this sense, we offer a method for learning to discretize the states of a continuous MDP. Prior literature has introduced algorithms for the same purpose. \citet{moore1994parti} introduced the Parti-Game algorithm, which uses a decision tree to dynamically partition a continuous state space based on the need for further exploration. That is, as data about the underlying environment is collected, state partitions are refined depending on a minimax score with respect to an adversary that prevents the learning algorithm from reaching the goal (and knows the current partitioning scheme). %If the current partition is a losing one, then its resolution is increased. In this way, over time, the partitioning will slowly become more fine grained in regions of the state space where it is needed. 
For Parti-Game to be applied, we must assume 1) the transition function is deterministic, 2) the MDP is goal-based and the goal state is known, and 3) a local greedy controller is available. 

\citet{feng2004dynamic} also make use of a tree-based approach---this time, kd-trees~\cite{friedman1976algorithm}---to dynamically partition a continuous MDP's state space into discrete regions. In contrast to Parti-Game, partitions are chosen based on value equivalence, thereby enabling a form of closure under the Bellman Equation. 

\citet{chapman1991input} study tree-based partitioning as a means of generalizing knowledge in RL, leading to the development of the G algorithm. The G algorithm constructs a data-dependent tree of Q-value partitions based on which Q value can adequately summarize different regions of the state space. Over time, the tree will grow to sufficiently represent the needed distinctions in states. %, but will pick up on any structure among Q values in the state space, thereby leading to generalization across similar states. 
Further work uses decision trees of different forms to partition complex (and often continuous) state spaces into discrete models~\cite{mccallum1996learning,uther1998tree}.
%\mnote{Is there something to be said about dimensionality, here? I feel like partigame would not work well in high dimensional continuous space. Maybe our approach wouldn't either... just wondering.} \dnote{I'm not sure. I don't think we can say anything meaningful.}

\citet{menashe2018state} also introduced an algorithm for abstracting continuous state, with the goal of inducing a small, tractable decision problem. They present RCAST (Recursive Cluster-based Abstraction Synthesis Technique), a technique for constructing a state abstraction that maps continuous states to discrete ones. Like the other tree-based methods discussed, RCAST uses kd-trees to partition the state space. The key insight at the core of RCAST is to partition based on the ability to predict different {\it factors} that characterize the state space. This is the main departure between the approach of RCAST and our own algorithm, which is not tailored to factored representations.~\citet{krose1992adaptive}, who present an adaptive quantization for continuous state spaces, to be used with simple RL algorithms like TD-learning.~\citet{lee2004adaptive} offer a similar technique based on vector quantization that performs partitioning based directly on data gathered by TD learning algorithm. \citet{whiteson2007adaptive} introduce a method for adaptive tile coding, to be used in value function approximation. \citet{liang2016state} presented evidence that ``shallow" learning algorithms can achieve competitive performance to many Deep RL algorithms in Atari games. Their approach constructs features that are well suited to the structure of Atari games, including properties like relative object and color locations. The main result of the work shows that with a well crafted set of features, even a simple learning algorithm can achieve competitive scores in Atari games. 

% Transferring state abstractions
We are not the first to explore transferring state abstractions. Indeed, \citet{walsh2006transferring} studied the process of transferring state abstractions across MDPs drawn from the same distribution. ~\citet{abel2018salrl} builds on this work by introducing PAC abstractions, which are guaranteed to retain their usefulness over a distribution of tasks. Notably, similar to our main theoretical result, PAC abstractions also retain value with respect to a distribution of MDPs with high probability based on the value loss of the abstraction family~\cite{abel2016near}. The crucial difference is that, again, PAC abstractions are tailored to discrete state spaces, and do not extend to continuous ones. \citet{cobo2011automatic} and ~\citet{cobo2012automatic} study methods for finding state abstractions based on a demonstrator's behavior. Similarly to our approach, their method is based on finding abstract states that can be used to {\it predict} what a demonstrator will do in those clusters. The key differentiating factor is that our algorithm targets continuous state spaces, while theirs focuses on discrete state spaces.

%\citet{singh1995reinforcement} study {\it soft} state aggregation, in which environmental states are not placed into discrete buckets, but are instead captured through a probabilistic abstraction function, $\tilde{\phi}: \mathcal{S} \rightarrow \Pr(\mathcal{S}_\phi)$.

\citet{majeed2018q} discussed Q-Learning convergence in non-Markov decision problems. In particular, Theorem 7 of their work reports that Q-Learning converges in some non-MDPs, under relatively mild assumptions. They go on to characterize the necessary and sufficient conditions for Q-Learning convergence, which does in fact extend to some non-MDPs. In this sense, their work builds on Theorem 4 from~\citet{li2006towards} which states that Q-Learning, using a policy-based abstraction (similar to what we learn), can sometimes converge to a policy that is sub-optimal in the original problem. Indeed, in experiments, we find this to not be the case -- it is still an open important question as to how certain abstractions effect both the estimation error and sample complexity of different learning algorithms.

A separate but equally relevant body of literature investigates learning state {\it representations} in the context of deep neural networks, typically for use in deep RL. For instance, \citet{jonschkowski2015learning} proposed learning state representations through a set of well chosen ontological priors, catered toward robotics tasks, including a simplicity prior and a causality prior (among others). These priors are then encoded into an optimization problem that seeks to jointly optimize over each of their prescribed properties. They conduct experiments in a continuous grid domain similar to Puddle World, comparing the use of {\it all} priors to each one individually, showcasing a consistent performance increase in learning with all priors. \cite{karl2016deep} developed a variational Bayes method for learning a latent state-space representation of a Markov model, given high dimensional observations. Critically, this state space is of a simple Markov model, and does not involve decision making or rewards, which are critical aspects of learning state representations in MDPs~\cite{oh2017value}. For a full survey of recent state representation schemes for deep RL, see \citet{lesort2018state}.

% Meta Learning/transfer learning.
Naturally, many active areas of recent literature explore transfer for RL. For instance, work on learning to learn~\cite{thrun1998learning}, or meta RL~\cite{finn2017model}, investigate how to explicitly improve sample efficiency across a collection of tasks. Recent work in meta RL has studied how to improve model-based RL~\cite{saemundsson2018meta} and how to learn a new RL algorithm explicitly~\cite{wang2016learning}. Other areas of research have explored transferring shaping functions~\cite{konidaris2006autonomous,konidaris2012transfer}, successor features~\cite{barreto2017successor}, skills~\cite{konidaris2007building,da2012learning,brunskill2014pac}, hierarchies~\cite{wilson2007multi,mehta2008transfer,tessler2017deep}, behavior~\cite{taylor2005behavior}, and transfer for deep RL~\cite{higgins2017darla,parisotto2015actor,teh2017distral}. For a survey of transfer in RL see~\citet{taylor2009transfer}.

% -----------------
% --- Algorithm ---
% -----------------
\section{Algorithm}
\label{section:explain_algorithm}
In this section, we formulate an optimization problem and propose an algorithm for learning state abstractions for MDPs with continuous states. Our approach builds on a type of abstraction studied by \citet{li2006towards} in the tabular setting. Specifically, the abstraction functions attempts to cluster together the states for which the optimal policy is similar and map these states to a common abstract state. Our goal will be to learn such an abstraction function on a set of training problems, and then transfer and use this representation for reinforcement learning in unseen problems.

We define a stochastic abstraction function $\phi: \mathcal{S}\mapsto\p(\FancyC)$ as a function that maps from ground states $s\in \FancyS$ to a probability distribution over abstract states $\FancyC$~\cite{singh1995reinforcement}. We focus on problems where $\FancyS$ is infinite, but $\FancyC$ is finite (problems with continuous state space and an abstraction that maps to a discrete one). Our goal is to find a state abstraction $\phi$ that enables simple learning algorithms to perform data-efficient reinforcement learning on new tasks. 

To learn an abstraction function, we introduce an objective that measures the probability of trajectories $\tau^{i}$ provided by our learned policy $\pi^{*}$. The goal will be to maximize this probability if we were to use the abstraction function $\phi$ and a policy, $\piPhi$, over abstract states. We formulate the optimization problem as follows:
\begin{equation*} \argmax_{\phi,\piPhi} \Pi_{i=1}^{M} \p(\tau^{i},\phi,\piPhi)=\argmax_{\phi,\piPhi} \sum_{i=1}^{M}\log  \p(\tau^{i},\phi,\piPhi)\label{original_problem}\ ,\end{equation*}
where:
\begin{eqnarray*}
\log \p(\tau^{i},\phi,\piPhi)&=&\log\Pi_{j=1}^{T({\tau}^i)}\pi(a^{\tau^i}_{j}|s_{j})\p(s^{\tau^i}_{j+1}|s^{\tau^i}_{j},a^{\tau^{i}}_{j})\\
&=&\sum_{j=1}^{T({\tau^i})} \log \sum_{c}\phi(c\mid s_{j}^{{\tau}^i})\pi^{\phi}(a_{j}^{{\tau}^i}\mid c)\\
&+&\sum_{j=1}^{T({\tau^i})}\log\p(s^{\tau^i}_{j+1}\mid s^{\tau^i}_{j},a^{\tau^i}_{j})\ .
\end{eqnarray*}
The second sum is not a function of the optimization variables, and could be dropped:
\begin{equation}
\begin{aligned}
&\argmax_{\phi,\piPhi} \sum_{i=1}^{M}\log \p(\tau^{i},\phi,\piPhi)\nonumber \\
&=\argmax_{\phi,\piPhi} \sum_{i=1}^{M}\sum_{j=1}^{T({\tau}^i)}\log \sum_{c}\phi(c\mid s_{j}^{{\tau}^i})\pi^{\phi}(a_{j}^{{\tau}^i}\mid c)
\end{aligned}
\end{equation}
More importantly, we consider the case where our training set consists of $K$ different MDPs. In this case, we solve for a generalization of the above optimization problem, namely:
\begin{equation*}
\argmax_{\phi,\piPhi}\sum_{k=1}^K \sum_{i=1}^{M}\sum_{j=1}^{T({\tau}^i)}\log \sum_{c}\phi(c\mid s_{j}^{{\tau}^i})\pi^{\phi}(a_{j}^{{\tau}^i}\mid c,k)\ .
\end{equation*}
If the solution to the optimization problem is accurate enough, then states that are clustered into a single abstract state generally have a similar policy in the MDPs used for training.

Although it is possible to jointly solve this optimization problem, for simplicity we assume $\pi^{\phi}$ is fixed and provided to the learner. We parameterize the abstraction function $\phi$ by a vector $\theta$ representing the weights of the network $\phi(\cdot|s;\theta)$. We use softmax activation to ensure that $\phi$ outputs a probability distribution.

A good setting of $\theta$ can be found by performing stochastic gradient ascent on the objective above, as is standard when optimizing neural networks \cite{lecun2015deep}:
$$\theta\!\leftarrow\!\theta + \alpha \nabla_{\theta} \sum_{k=1}^K\sum_{i=1}^{M}\sum_{j=1}^{T({\tau}^i)}\log \sum_{c}\phi(c\mid s_{j}^{{\tau}^i};\theta)\pi^{\phi}(a_{j}^{{\tau}^i}\mid c,k)$$
In our experiments we used the Adam optimizer \cite{kingma2014adam} which can be thought of as an advanced extension of vanilla stochastic gradient ascent.

% ------------
% -- Theory --
% ------------
\section{Theory}
Given a learned state abstraction, it is natural to ask if the policy over the learned abstract state space can achieve a reasonable value. In this section we answer this question positively by providing a bound on the value loss of the abstract policy relative to the demonstrator policy. %This bound will hold with high probability.

We now outline our proof at a high level. The first step shows the abstraction learning algorithm described in Section (\ref{section:explain_algorithm}) has bounded $\elone{\cdot}$ policy difference on expectation under states in the training set. We then use Rademacher complexity to generalize this policy difference to any set of states drawn from the same distribution. Given the generalization bound, we use the following lemma to conclude that the abstraction has bounded value loss.
%%%%%%%
% --- E[KL] Implies E[V] ---
%%%%%%%
\begin{lemma}(Corollary of Lemma 2 in~\citet{abel2019rlit})
Consider two stochastic policies, $\pi_1$ and $\pi_2$ on state space $\FancyS$, and a fixed probability distribution over $\FancyS$, $p(s)$. If, for some $k \in \mathbb{R}_{\geq 0}$:
    \begin{equation*}
     \underset{p(s)}{\mathbb{E}}\left[||\pi_1(a \mid s) - \pi_2(a \mid s)||_1 \right] \leq k,
    \end{equation*}
    then:
    \begin{equation*}
         \underset{p(s)}{\mathbb{E}}\left[ V^{\pi_1}(s) - V^{\pi_2}(s)\right] \leq \frac{k\textsc{RMax}}{1-\gamma}.
    \end{equation*}
 \label{lem:l1_val_bound}
\end{lemma}
%%%%%%%%%%%%%%%%%%%%%
To satisfy the assumption of the lemma, we first show that the objective function can be rewritten using Kullback-Leibler (KL) divergence. For the rest of this section, we assume that each cluster $c_a$ assigns probability 1 to the action $a$, allowing us to simplify notation:%$\sum_{c}\phi(c|s)\pi^{\phi}(a|c)=\phi(c_a|s)$:
\begin{eqnarray*}
\lefteqn{\argmax_{\phi} \sum_{i=1}^{M}\sum_{j=1}^{T({\tau}^i)}\log \sum_{c}\phi(c\mid s^{{\tau}^i}_j)\piPhi(a^{{\tau}^i}_j\mid c)}\\
%&=&\argmax_{\phi} \expected{s\sim d^{\pi^{*}},a\sim \pi^{*}}{\log \phi(c_a\mid s)}\\
&=&\argmax_{\phi}\big\{\expected{}{\log \phi(c_a\mid s)}-\expected{}{\log \pi^{*}(a\mid s)}\big\}\\
&=&\argmax_{\phi} \expected{}{\log\frac{ \phi(c_a\mid s)}{\pi^{*}(a|s)}}\\
&=&\argmax_{\phi} \expected{s\sim d^{\pi^{*}}}{K\!L\big(\pi^{*}(\cdot\mid s)||\phi(\cdot\mid s)\big)}\ .
\end{eqnarray*}
Now, suppose that our training procedure has a bounded KL loss formalized below:
$$\frac{1}{n}\sum_{j=1}^{n}\sqrt{2 K\!L\big(\pi^{*}(\cdot\mid s_j)||\phi(\cdot\mid s_j)\big)}\leq \Delta.$$
Using Pinsker's inequality, we have:
$$\frac{1}{n}\sum_{j=1}^{n}\elone{\pi^{*}(\cdot\mid s_j)-\phi(\cdot\mid s_j)}\leq \frac{\Delta}{2}.$$
That is, we have a bounded $\elone{\cdot}$ policy difference. Critically, this bound is only valid given set of states seen in the training data, but to leverage Lemma~\ref{lem:l1_val_bound}, we need to bound the generalization error. We get this result by introducing a theorem. Lemma below is used in the theorem.
\begin{lemma}(Corollary 4 in \cite{maurer2016vector}) Assume a set of vector valued functions $\FancyG=\{g:\FancyX \mapsto \mathbb{R}^{m}\}$. Assume $n$ L-Lipschitz functions $l_j\ \forall j\in \{1,...,n\}$. Then, the following inequality holds:
\begin{equation*}
\begin{aligned}
 &\expected{X_j,\sigma_j}{\sup_{g\in\FancyG}\frac{1}{n}\sum_{j=1}^n \sigma_j l_j\big(g(X_j)\big)}\\
 &\leq\sqrt{2}L \expected{X_j,\sigma_{ji}}{\sup_{g\in\FancyG}\frac{1}{n}\sum_{j=1}^n\sum_{i=1}^m \sigma_{ji} g(X_j)_i} \ .
\end{aligned}
\end{equation*}
\label{maurer_lemma}
\end{lemma}
\begin{theorem*}
With probability at least $1-\delta$, for any $\delta \in (0,1)$:
\begin{equation*}
\begin{aligned}
    & \expected{s}{\elone{\big(\pi^{*}(\cdot|s) - \phi(\cdot|s)\big)}}\leq \frac{\Delta}{2}+
    2\sqrt{2}\rademacher{\Phi}+\sqrt{\frac{2\ln{\frac{1}{\delta}}}{n}}
\end{aligned}
\end{equation*}
\end{theorem*}
\begin{proof}
We build on techniques provided by \citet{bartlett2002rademacher}. First, note that $\forall \phi$ we have:
\begin{equation}
\begin{aligned}
 & \expected{s}{\elone{\pi^{*}(\cdot|s)-\phi(\cdot|s)}}-\frac{1}{n}\sum_{j=1}^{n}\elone{\pi^{*}(\cdot|s_j)-\phi(\cdot|s_j)}\leq\\
 & \underbrace{\sup_{\phi\in\Phi}\{\expected{s}{\elone{\pi^{*}(\cdot|s)-\phi(\cdot|s)}}\!-\!\frac{1}{n}\!\sum_{j=1}^n\! \elone{\pi^{*}(\cdot|s_j)-\phi(\cdot|s_j)}\}} _{:=\Psi(s_1,...,s_n)}
\end{aligned}
\label{eq:sup_property}
\end{equation}
We can bound the expected value of $\Psi$:
\begin{equation}
    \expected{}{\Psi}\leq 2\sqrt{2}\rademacher{\Phi}
    \label{eq:bound_expected_Psi}
\end{equation}
We do so as follows:
\begin{eqnarray*}
     &&\expected{}{\Psi}= \expected{s_j}{\sup_{\phi\in\Phi}\expected{s}{\frac{1}{n} \sum_{j=1}^n \elone{\pi^{*}(\cdot|s)-\phi(\cdot|s)}\\
     &&- \elone{\pi^{*}(\cdot|s_j)-\phi(\cdot|s_j)}}}\\
     &&= \expected{s_j}{\sup_{\phi\in\Phi}\expected{s'_j}{\frac{1}{n} \sum_{j=1}^n \elone{\pi^{*}(\cdot|s'_j)-\phi(\cdot|s'_j)}-\\ & &\elone{\pi^{*}(\cdot|s_j)-\phi(\cdot|s_j)}}},
\end{eqnarray*}
since $s'_j$ and $s_j$ are distributed similarly. %dnote: just trying to save space.
\begin{eqnarray*}
     && \leq \expected{s_j,s'_j}{\sup_{\phi\in\Phi} \frac{1}{n}\sum_{j=1}^n \elone{\pi^{*}(\cdot|s'_j)-\phi(\cdot|s'_j)}\\
     &&-\elone{\big(\pi^{*}(\cdot|s_j)-\phi(\cdot|s_j)}}\\
     && \text{(Due to Jansen's inequality)}\\
      &&= \expected{s_j,s'_j,\sigma_{ji}}{\sup_{\phi\in\Phi} \frac{1}{n} \sum_{j=1}^n \sigma_j \Big(\elone{\pi^{*}(\cdot|s'_j)-\phi(\cdot|s'_j)}-\\ && \elone{\pi^{*}(\cdot|s_j)-\phi(\cdot|s_j)}\Big)}\\
      &&\text{(Due to Rademacher variables uniform from $\{\pm 1\}$})\\
      &&=2 \expected{s_j,s'_j,\sigma_{ji}}{\sup_{\phi\in\Phi} \frac{1}{n} \sum_{j=1}^n \sigma_j \elone{\pi^{*}(\cdot|s'_j)-\phi(\cdot|s'_j)}}\\
      &&\leq2\sqrt{2} \expected{s_j,\sigma_{ji}}{\sup_{\phi\in\Phi} \frac{1}{n} \sum_{j=1}^n\sum_{i=1}^{|A|} \sigma_{ji} \phi(a_i|s_j)}\\
      && \text{(Due to Lemma (\ref{maurer_lemma}) and $\elone{\cdot}$ being 1-Lipschitz)}\\
      && =2\sqrt{2}\rademacher{\Phi} \ .
\end{eqnarray*}
Further, it is easy to show that $\Psi$ satisfies
$$|\Psi(s_1,...,s_j,...s_n)-\Psi(s_1,...,s'_j,...s_n)|\leq\frac{2}{n}.\ $$ Applying MacDiarmid's inequality, for any $\delta \in (0,1)$:
\begin{equation}\Pr(\Psi\leq \expected{}{\Psi}+\sqrt{\frac{2\ln{\frac{1}{\delta}}}{n}})\geq 1- \delta\ .
\label{eq:MacDiarmid}
\end{equation}
Combining (\ref{eq:MacDiarmid}) with (\ref{eq:bound_expected_Psi}), we can write:
$$\Pr(\Psi\leq 2\sqrt{2}\rademacher{\Psi}+\sqrt{\frac{2\ln{\frac{1}{\delta}}}{n}})\geq 1- \delta\ .$$
Now recall from $(\ref{eq:sup_property})$ that:
\begin{equation*}
\begin{aligned}
 & \expected{s}{\elone{\pi^{*}(\cdot|s)- \phi(\cdot|s)}}\leq\frac{1}{n}\sum_{j=1}^{n}\elone{\pi^{*}(\cdot|s_j)-\phi(\cdot|s_j)}+\Psi,
\end{aligned}
\end{equation*}
we can conclude the proof by claiming that
\begin{equation*}
\begin{aligned}
 & \expected{s}{\elone{\pi^{*}(\cdot|s)-\phi(\cdot|s)}}\leq\\ 
 & \frac{1}{n}\sum_{j=1}^{n}\elone{\pi^{*}(\cdot|s_j)-\phi(\cdot|s_j)}+2\sqrt{2}\rademacher{\Phi}+\sqrt{\frac{2\ln{\frac{1}{\delta}}}{n}}\ .
\end{aligned}
\end{equation*}
with probability at least $1-\delta$.
\end{proof}
It is clear to see that, given the above bound, Lemma \ref{lem:l1_val_bound} ensures bounded value loss with probability at least $1-\delta$.

% -------------------
% --- Experiments ---
% -------------------
\section{Experiments}

We now present our experimental findings. At a high level, we conducted two types of experiments:
\begin{enumerate}
    \item Single Task: We collected an initial data set $D_{train}$ to be used to train the state abstraction $\phi$ based on MDP $M$. Then, we evaluate the performance of tabular Q-Learning on $M$, given this state abstraction. Since each $M$ we test with has continuous state, tabular Q-Learning cannot be applied. However, we assess the performance of tabular Q-learning given $\phi$. These experiments provide an initial sanity check as to whether the state abstraction can facilitate learning at all.
    
    \item Multi-Task: We next considered a collection of MDPs $\{M_1, \ldots, M_n\}$. We collected an initial data set $D_{train}$ of $(s,a)$ tuples from a (strict) subset of MDPs in the collection. We used $D_{train}$ to construct a state abstraction $\phi$, which we then gave to Q-Learning to learn on one or many of the remaining MDPs. Critically, we evaluate Q-Learning on MDPs {\it not} seen during the training of $\phi$.
\end{enumerate}

For each of the two experiment types, we evaluate in three different domains, Puddle World, Lunar Lander, and Cart Pole. Open-source implementation of Lunar Lander and Cart Pole are available on the web \cite{gym}. An implementation of Puddle World is included in our code. Full parameter settings and other experimental details are available in our anonymized code, which we make freely available for reproduction and extension.\footnote{\url{https://github.com/anonicml2019/icml_2019_state_abstraction}}

%%%%%%%%%%%
% Single Task Learning Curves
\begin{figure*}[h]
    \centering
    %%%%% VISUALS
    % Puddle
   \hspace{7mm} \subfloat[Puddle World~\label{fig:puddle}]{\includegraphics[width=0.24\textwidth]{figures/puddle/puddle.jpg}} \hspace{12mm}
   % Lunar
    \subfloat[Lunar Lander~\label{fig:lunar}]{\includegraphics[width=0.24\textwidth]{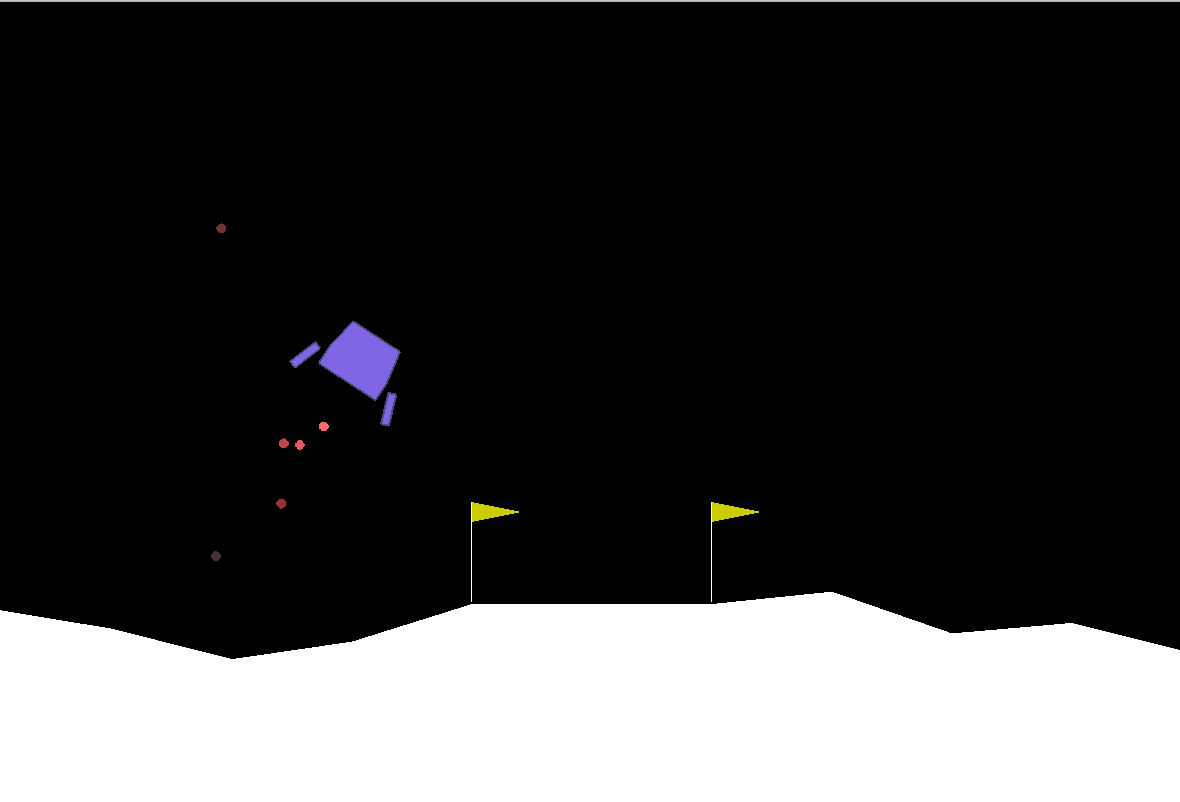}}\hspace{8mm}
    % Cartpole
    \subfloat[Cart Pole]{\includegraphics[width=0.24\textwidth]{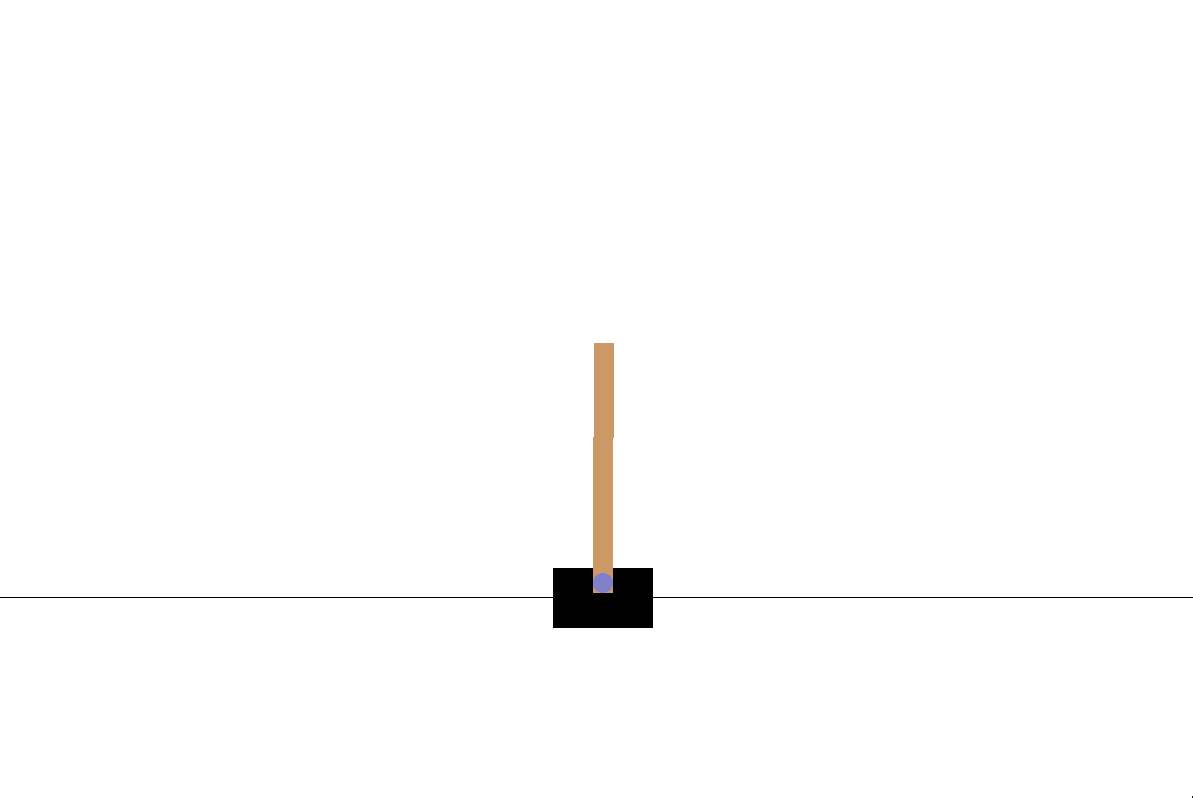}} \\
    
    %%%%% SINGLE TASK RESULTS
    % Puddle
    \subfloat[Single Task Puddle World \label{fig:puddle_results}]{\includegraphics[width=0.3\textwidth]{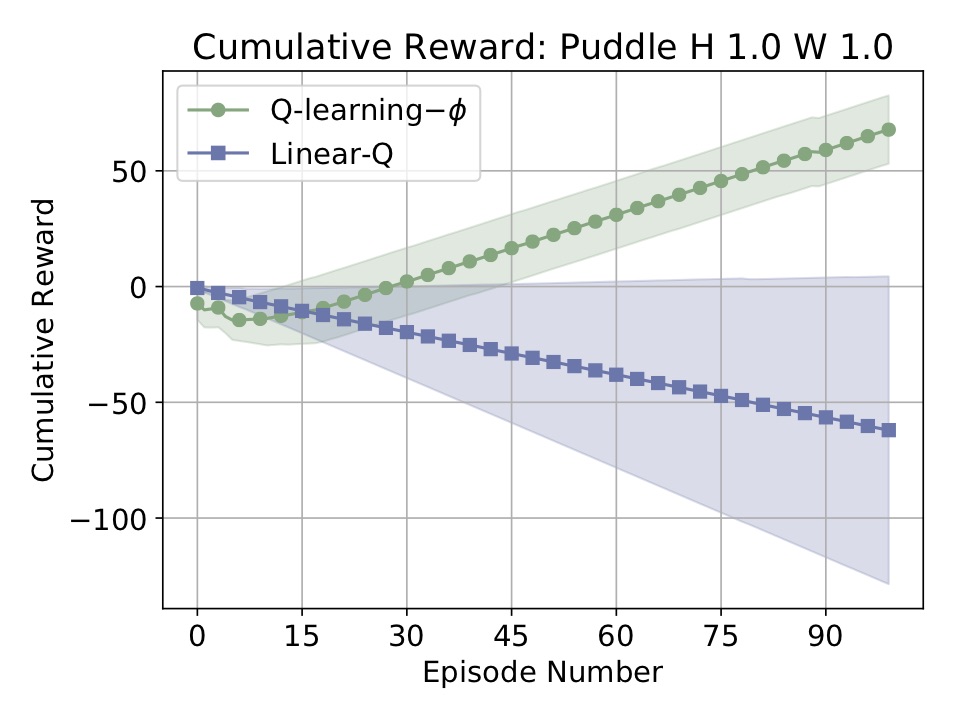}}
    % Lunar
    \subfloat[Single Task Lunar Lander \label{fig:lunar_results}]{\includegraphics[width=0.3\textwidth]{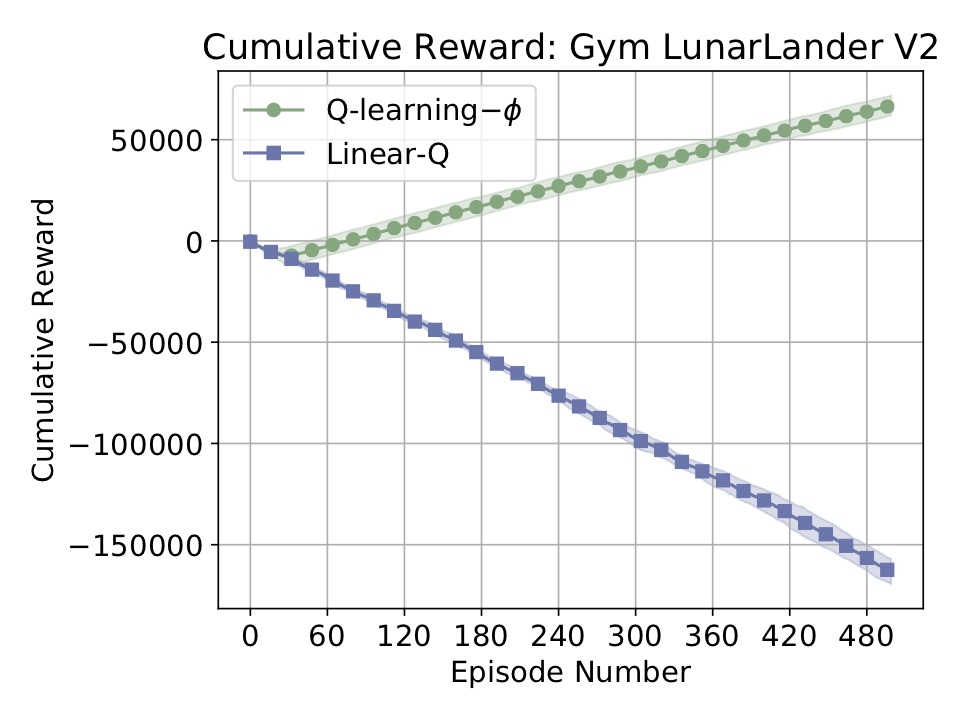}}
    % Cart Pole
    \subfloat[Single Task Cart Pole \label{fig:cartpole_results}]{\includegraphics[width=0.3\textwidth]{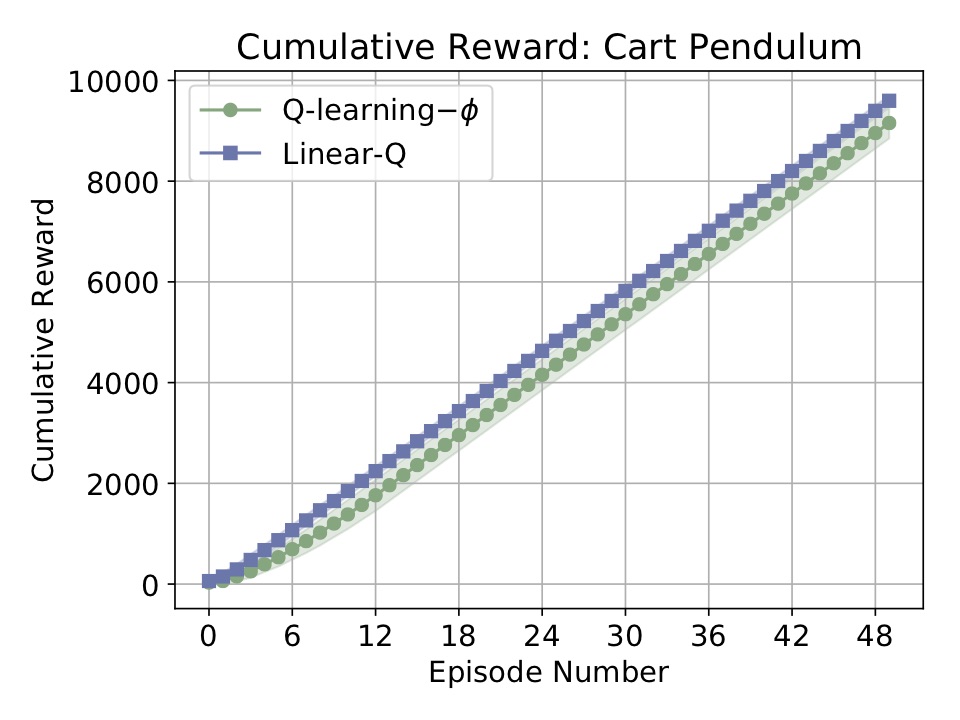}} \\
    
    %%%%%%%%%%%
    % Transfer Learning Curves
    % Puddle
    \subfloat[Puddle World Transfer \label{fig:puddle_transfer}]{\includegraphics[width=0.3\textwidth]{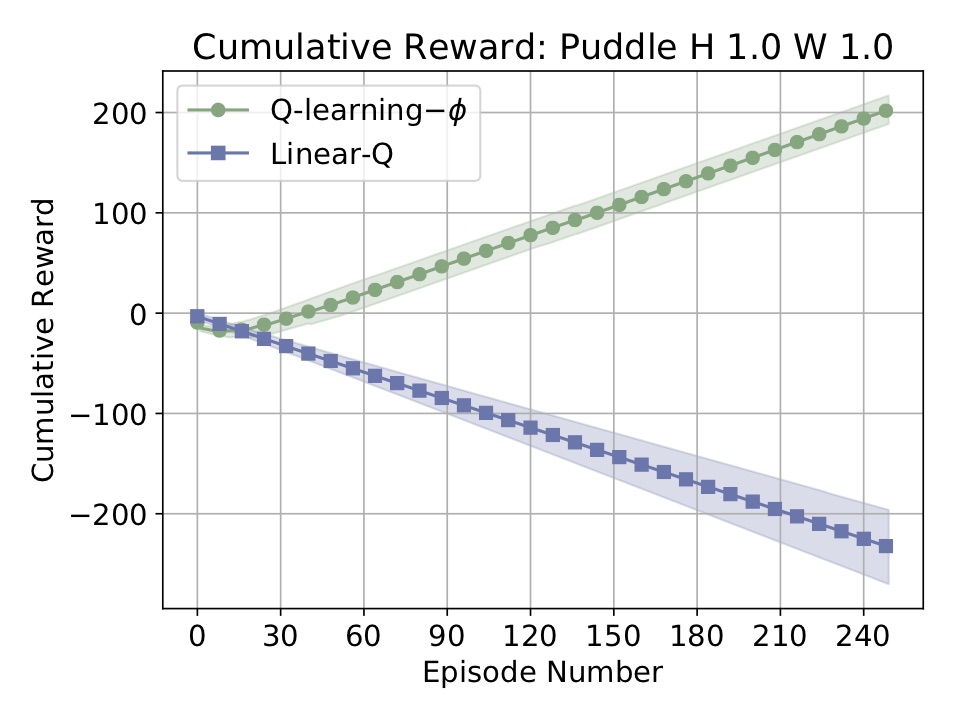}}
    % Lunar
    \subfloat[Lunar Lander Transfer \label{fig:lunar_transfer}]{\includegraphics[width=0.3\textwidth]{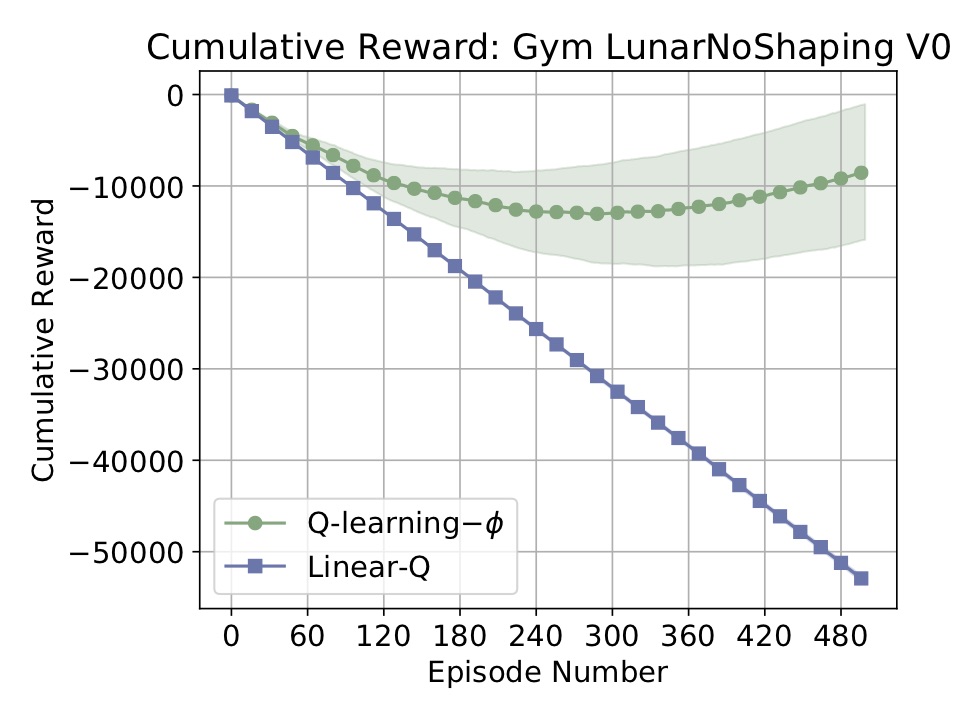}}
    % Cart Pole
    \subfloat[Cart Pole Transfer \label{fig:cartpole_transfer}]{\includegraphics[width=0.3\textwidth]{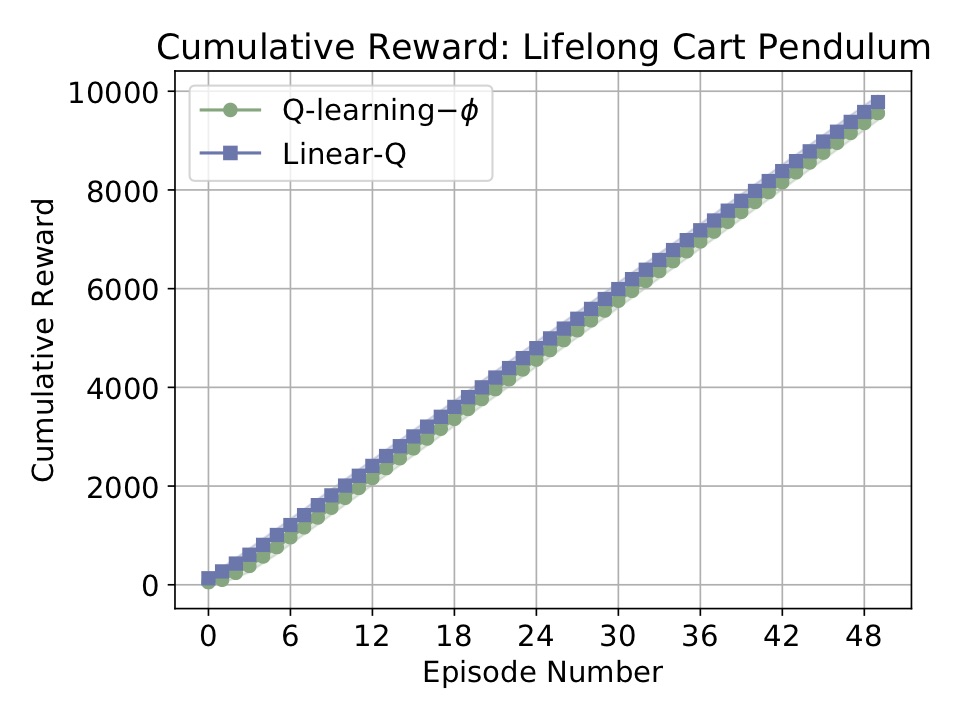}}
    \caption{Learning curves for the single task experiments (top) and the transfer experiments( bottom). Each line indicates the average cumulative reward received by the agent, reported with 95\% confidence intervals.}
    \label{fig:transfer_results}
\end{figure*}

\subsection{Puddle World}
% Puddle Experiment details..
Our first experiment used the Puddle World MDP~\cite{boyan1995generalization}, a continuous grid world in which states are represented by two coordinates, $x \in [0:1]$ and $y \in [0:1]$. The agent is initialized at (0.25, 0.6), and is tasked with getting within .0025 of the goal location, which is placed at (1.0,1.0). Two large rectangles are placed in the domain, which produce $-1$ reward for any time-step in which the agent is in either rectangle. All other transitions receive 0 reward, except for transitions into the goal state, which yield 1 reward. The agent is given four actions, \texttt{up}, \texttt{down}, \texttt{left}, and \texttt{right}. Each action moves the agent .05 in the given direction with a small amount of Gaussian noise.% noise added to the magnitude of the movement. The noise is sampled from the standard normal with $\mu = 0 $ and $\sigma^2 = 1$, where the final noise is divided by 100 to scale it appropriately to the domain.

% More Puddle Details.
In the single task Puddle World experiment, we trained the abstraction based on 4000 sampled $(s,\pi_E(s),r,s')$ quadruples from the puddle instance pictured in Figure~\ref{fig:puddle} with only goal $G_1$ active, and $\pi_E$ the training policy. The samples are drawn from $U(x,y)$, the joint uniform probability distribution over $x$ and $y$. Notably, since the domain is continuous, the learning agent will necessarily find itself in states it did not see during training time. We experiment with tabular Q-Learning paired with the abstraction, denoted Q-Learning-$\phi$ (green), and Q-Learning with a linear function approximator, denoted Linear-Q (blue). We set $\epsilon = 0.1$ and $\alpha = 0.005$ for both algorithms. It is worth noting that since the training of $\phi$ takes place in the same domain that we test Q-Learning-$\phi$, we should anticipate that the resulting pair learn quite well (assuming the $\phi$ is capable of supporting good learning at all). %In this sense, this experiment serves as a sanity check.% to ensure that the algorithm is working properly, before we move to the more complex experiments involving transfer.

% Puddle Results.
Results for the single task case are presented in Figure~\ref{fig:puddle_results}. The figures present the cumulative reward averaged over all 25 runs of the experiment, with $95\%$ confidence intervals. As expected, we find that Q-Learning-$\phi$ consistently learns to navigate to the goal in only a few episodes. Conversely, the linear approach fails to reliably find the goal, but does occasionally learn to avoid running into the puddle, resulting in high variance performance.

% Multi task set up.
In the multi-task Puddle World experiment, we train $\phi$ using three out of the four possible goals (each goal defines an independent MDP), located in the four corners. The held-out goal is chosen uniformly randomly from the four possible goals. We leave one goal out of training to be used for testing. All parameters are set as in the single task case, except that the agents now learn for 250 episodes instead of 100. The abstraction learning algorithm was given a budget of 81 abstract states.

% Multi-task results.
Results are presented in Figure~\ref{fig:puddle_transfer}. Notably, the only positive reward available comes from reaching the goal, so we can conclude from the low-variance positive slope of the learning curve that Q-Learning with $\phi$ still learns to reach the goal reliably in only a few episodes. Moreover, Q-learning with linear function approximation fails to ever reliably find the goal.

% --- Lunar ---
\subsection{Lunar}
% Lunar Lander.
We next experiment with Lunar Lander, pictured in Figure~\ref{fig:lunar}. The agent controls the purple ship by applying thrusters on exactly one of the left, right, bottom, or no sides of the ship. The state is characterized by 8 state variables, 2 dimensional position and velocity information, angle and angular velocity, and two boolean flags that are active only if the corresponding leg touches a surface. The agent receives +100 reward when it lands or -100 when it crashes, and an additional +10 reward for each leg of the ship touching the ground. Critically, the agent receives -.3 reward every timestep it uses a thruster for each of the three thrusters available. The default reward signal in gym implementation includes a shaped reward that gives the agent additional positive reward the closer it gets to the landing zone, positioned at (0,0) on the screen (the flat area with the flags). In each run, the mountainous terrain to the left and right of the landing zone changes.

% Single task description.
In the single task case, we train the state abstraction based on 10,000 $(s,\pi_E(s),r,s')$ quadruples, sampled according to the training policy $\pi_E$ (which reliably either lands or gets high shaped reward). We then give the resulting abstraction (which maps from the continuous 8-dimensional states to a discrete state space), to tabular Q-Learning and evaluate its learning performance.

% Results for single task.
Results are presented in Figure~\ref{fig:lunar_results}. Notably, again, Q-Learning with $\phi$ learns a policy that reliably lands the ship in around 50 episodes. In contrast, the linear Q-Learning agent fails to ever learn a reasonable policy even after collecting 8 times the data. The results suggest that again, $\phi$ is capable of supporting sample efficient learning of a high value policy.

In the transfer case, we train $\phi$ similarly to the single task case. Then, we define a new instance of the Lunar problem with {\it no} shaped reward. In this new, non-shaped variant, the agent receives $-.3$ for every use of a thruster, $-100$ for crashing, $+100$ for landing safely, and $+10$ anytime a leg touches the ground. This variant of Lunar is extremely challenging, and to our knowledge, no known algorithm has solved this problem from scratch.

% Transfer results.
Learning curves are presented in Figure~\ref{fig:lunar_transfer}, and a plot of the successful landing rate over time is also provided in Figure~\ref{fig:lunar_success}. Here we find the strongest support for the usefulness of $\phi$: Q-Learning-$\phi$ learns to land the ship more than half of the time after around 250 episodes. By the end of learning (500 episodes), Q-Learning with $\phi$ has learned to land the ship four out of every five trials.

\begin{figure}
    \centering
    \subfloat[Lunar Success Rate~\label{fig:lunar_success}]{\includegraphics[width=0.48\columnwidth]{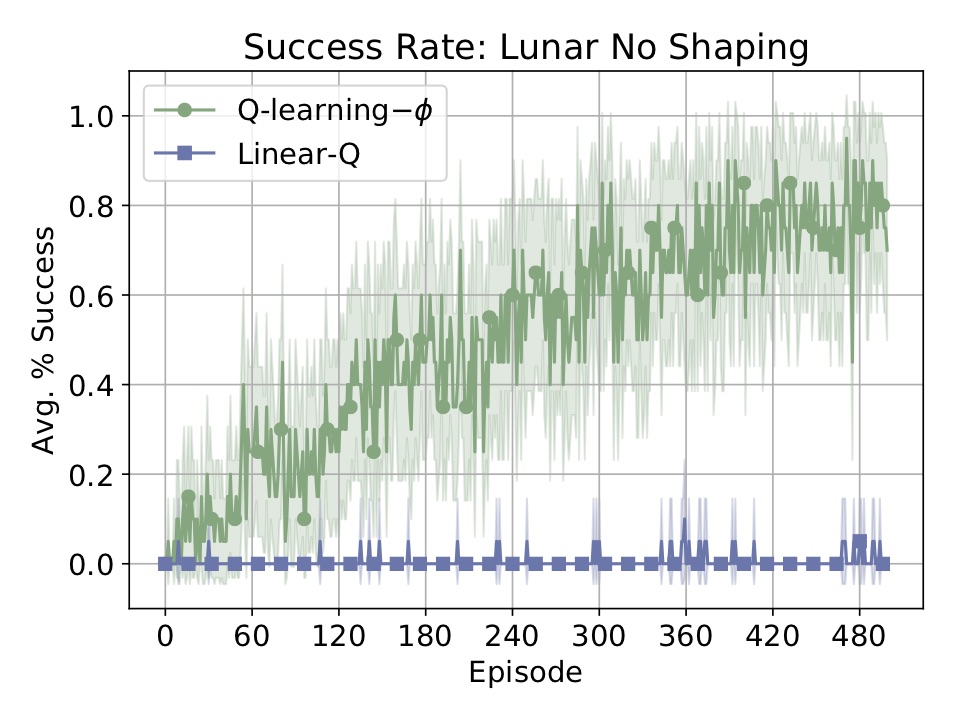}}
    \subfloat[Effect of $|D|$ on RL~\label{fig:puddle_training_data}]{\includegraphics[width=0.48\columnwidth]{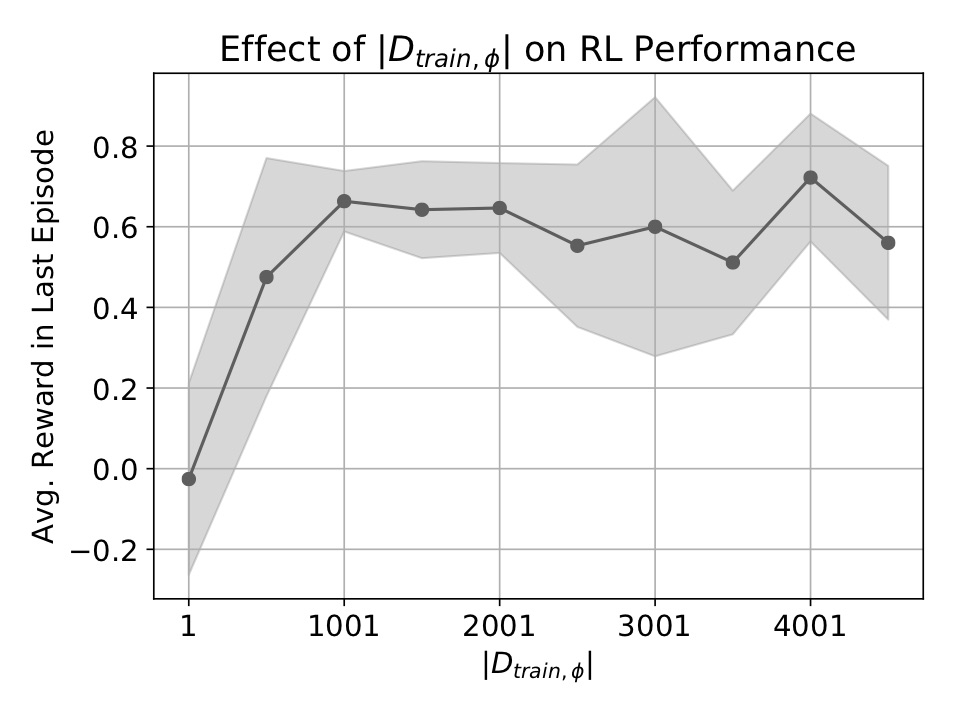}}
    \caption{The average rate of successful landings over time in Lunar Lander without shaping (left) and the effect of the number of training data on RL performance in Puddle World (right). In the right plot, a point indicates the cumulative reward received by Q-Learning by the end of learning with the $\phi$ trained on the data set of the given size, averaged over 10 runs of the experiment, reported with 95\% confidence intervals.}
    \label{fig:joint_lunar_success_training_size}
\end{figure}

% --- Cart Pole ---
\subsection{Cart Pole}
% Cart Pole.

Cart Pole is a classical control problem first studied by~\citet{widrow1964pattern} in which an agent must learn to balance a pole upright by moving a cart at the base of the pole along a horizontal track. The state of the MDP consists of four variables: $(x, \vec{x}, \omega, \vec{\omega})$, denoting the position, velocity, angle, and angular velocity respectively. The agent is given only two actions: move left and move right. The reward signal is +1 when the pole is balanced ($\omega \in (-\pi/9, \pi/9)$, and otherwise -10. When the pole moves outside of the positive reward region, the agent receives -10 reward and the MDP terminates. 

We train $\phi$ based on a data set of 1000 $(s,\pi_E(s),r,s')$ quadruples collected from the training policy, with a learning rate of $0.001$. Then, we give Q-Learning $\phi$ to learn on the same MDP for 50 episodes, with a max of 200 steps per episode. We repeat the experiment 20 times and report average learning curves and 95\% confidence intervals.

Results are presented in Figure~\ref{fig:cartpole_results}. Here we find that LinearQ is able to efficiently learn a high value policy, along with Q-Learning-$\phi$. From the set of single task experiments, we conclude that $\phi$ is at least sufficient to support tabular learning in continuous state MDPs in which $\phi$ was trained.

% Transfer.
In the multi-task case, we train $\phi$ with the same parameters as in the single task experiment. Then, we build a collection of 5 test MDPs where we change gravity from $9.8 m/s^2$ to each of 5.0, 6.0, 8.0, and 12.0 $m/s^2$. We then evaluate Q-Learning-$\phi$ on its average performance across each of these 5 test MDPs. When learning begins, we sample one MDP from the test set and let the agent interact with the sampled MDP for 200 episodes. After the last episode, the agent samples a new MDP uniformly from the same set of 5 and learns, repeating this process for 20 rounds.

% Multi task results.
Results are presented in Figure~\ref{fig:cartpole_transfer}. The learning curve is very abrupt; after only a handful of episodes, both learning algorithms again reliably learns to balance the pole, regardless of how much gravity has changed from the original problem. Here we find support for the use of a linear approximator in some settings, but naturally, as the environment becomes complex, the linear approach fails (as in the puddle and lunar experiments)

%\mnote{So, the use of uniform samplng is definitely unsettling. I wonder if we should cite some of the other work that found the need to do that. I think Satinder had a paper on policy iteration or something like that. Hmmm, that's not a good enough description. Ok, I think it was this paper: http://papers.nips.cc/paper/1531-finite-sample-convergence-rates-for-q-learning-and-indirect-algorithms.jpg .}
% dnote: Okay, we removed the sampling experiment.

% Experiment description.
As a final experiment, we explore the effect of the size of the training data set used to train $\phi$ on the performance of the downstream RL task that $\phi$ will be used for. Our goal is to investigate how many samples are sufficient for the learned state abstraction to reliably produce good behavior in practice. In the experiment, we trained $\phi$ for $N = 1$ up to $N = 4501$ samples (by increments of 500) to convergence. We then gave the resulting $\phi$ to Q-Learning and let it learn in Puddle World with all parameters set as in the single task Puddle experiment, and report the cumulative reward of Q-Learning's final episode.

% Results of |D| experiment.
Results are presented in Figure~\ref{fig:puddle_training_data}. As expected, when the data set used to train $\phi$ only contains a single point, RL performance decays substantially, with Q-Learning receiving 0 reward on average. However, somewhat surprisingly, with only 501 data points, the computed $\phi$ is sufficient for Q-Learning to obtain an average cumulative reward of 0.5. Such an average score indicates that, in about half of the runs, Q-Learning-$\phi$ reliably found the goal, and, in the other half, it avoided the puddle.

% ------------------
% --- Conclusion ---
% ------------------
\section{Conclusion and Future Work}
 We introduced an algorithm for state abstraction and showed that transferring the learned abstraction to new problems can enable simple RL algorithms to be effective in continuous state MDPs. We further studied our abstraction learning process through the lens of statistical learning theory and Rademacher complexity, leading to the result that our  learning algorithm can represent high value abstract policies. We believe that this work takes an important step towards the problem of automatic abstraction learning in RL.

One avenue for future research extends our approach to continuous action spaces; in this case, we can no longer leverage the fact that the number of actions are finite. Moreover, learning with continuous state {\it and} action is notoriously difficult and ripe for abstraction to make a big impact. Additionally, there are open questions on learning abstractions in a lifelong setting, in which agents continually refine their abstractions after seeing new problems. Lastly, it is unknown as to how different kinds of abstractions effect the sample complexity of RL~\cite{kakade2003sample}. Thus, we foresee connections between continuous state exploration~\cite{pazis2013pac} and learning abstractions that can lower the sample complexity of RL.

\bibliography{abstraction}
\bibliographystyle{icml2019}

\end{document}